\documentclass[12pt]{article}

\usepackage{times}
\usepackage{amsfonts}
\usepackage{amsmath}
\usepackage[psamsfonts]{amssymb}
\usepackage{latexsym}
\usepackage{color}
\usepackage{graphics}
\usepackage{enumerate}
\usepackage{amstext}
\usepackage{url}
\usepackage{epsfig}
\usepackage{graphicx} 
\usepackage{amsmath, amssymb, amsthm,wasysym}
\usepackage{color}
\usepackage{times}

\usepackage{graphicx}

\newtheorem{theorem}{Theorem}   
\newtheorem{lemma}[theorem]{Lemma}
\newtheorem{corollary}[theorem]{Corollary}
\newtheorem{definition}{Definition}
\newtheorem{remark}{Remark}

\newcommand{\figref}[1]{Figure~\ref{#1}}
\newcommand{\secref}[1]{Section~\ref{#1}}
\newcommand{\thmref}[1]{Theorem~\ref{#1}}
\newcommand{\lemref}[1]{Lemma~\ref{#1}}

\newcommand{\corref}[1]{Corollary~\ref{#1}}

\newcommand{\by}{\boldsymbol{y}}

\newcommand{\fork}{\textsc{fork}}

\newcommand{\scO}{\mathcal{O}}

\newcommand{\scP}{\mathcal{P}}
\newcommand{\scY}{\mathcal{Y}}

\newcommand{\opt}{\textsc{opt}}

\newcommand{\Psiinv}{\Psi^{*}}

\newcommand{\dt}{\displaystyle}

\newcommand{\sel}{\textsc{sel}}

\newcommand{\pred}{\textsc{pred}}

\newcommand{\lb}{\mbox{\textsc{lb}}}

\newcommand{\scT}{\mathcal{T}}
\newcommand{\yhat}{\widehat{y}}
\newcommand{\treeopt}{\textsc{TreeOpt}}

\newcommand{\bto}{\textsc{bto}}
\newcommand{\inbd}[1]{\underline{\partial{#1}}}

\begin{document}

\title{Active Learning on Trees and Graphs}

\author{
Nicol\`o Cesa-Bianchi\\ 
Dipartimento di Informatica, Universit\`a degli Studi di Milano, Italy\\
\texttt{nicolo.cesa-bianchi@unimi.it}
\and
Claudio Gentile\\ 
DiSTA, Universit\`a dell'Insubria, Italy\\
\texttt{claudio.gentile@uninsubria.it}
\and
Fabio Vitale\\ 
Dipartimento di Informatica, Universit\`a degli Studi di Milano, Italy\\
\texttt{fabio.vitale@unimi.it}
\and
Giovanni Zappella\\
Dipartimento di Matematica, Universit\`a degli Studi di Milano, Italy\\
\texttt{giovanni.zappella@unimi.it}
}

\maketitle

\begin{abstract}
We investigate the problem of active learning on a given tree
whose nodes are assigned binary labels in an adversarial way.
Inspired by recent results by Guillory and Bilmes, we characterize
(up to constant factors) the optimal placement of queries
so to minimize the mistakes made on the non-queried nodes.
Our query selection algorithm is extremely efficient, and
the optimal number of mistakes on the non-queried nodes is
achieved by a simple and efficient mincut classifier.
Through a simple modification of the query selection algorithm
we also show optimality (up to constant factors) with respect
to the trade-off between number of queries and number of mistakes
on non-queried nodes.
By using spanning trees, our algorithms can be efficiently applied
to general graphs, although the problem of finding optimal and
efficient active learning algorithms for general graphs remains open.
Towards this end, we provide a lower bound on the number of
mistakes made on arbitrary graphs by any active learning
algorithm using a number of queries which is up to a constant 
fraction of the graph size.
\end{abstract}

\section{Introduction}
\label{s:intro}
%
The abundance of networked data in various application domains (web, social
networks, bioinformatics, etc.) motivates the development of scalable and accurate
graph-based prediction algorithms. An important topic in this area is the graph
binary classification problem:
Given a graph
with unknown binary labels on its nodes,
the learner receives the labels on a subset of the nodes (the training set)
and must predict the labels on the remaining vertices. This is typically done
by relying on some notion of label regularity depending on the graph topology,
such as that nearby nodes are likely to be labeled similarly.
Standard approaches to this problem predict with the assignment of labels
minimizing the induced cutsize~(e.g., \cite{BC01,BLRR04}), or by binarizing the assignment
that minimizes certain real-valued extensions of the cutsize function~(e.g., \cite{ZGL03,BMN04,BDL06}
and references therein).

In the active learning version of this problem the learner is allowed to choose
the subset of training nodes.
Similarly to standard feature-based learning, one expects active methods to
provide a significant boost of predictive ability compared to a noninformed
(e.g., random) draw of the training set.
The following simple example provides some intuition of why this could happen
when the labels are chosen by an adversary, which is the setting
considered in this paper.
Consider a ``binary star system'' of two star-shaped graphs whose centers are
connected by a bridge, where one star is a constant fraction bigger than the other.
The adversary draws two random binary labels and assigns the first label to
all nodes of the first star graph, and the second label to all nodes of the
second star graph. Assume that the training set size is two.
If we choose the centers of the two stars and predict with a mincut strategy,\footnote
{
A mincut strategy considers all labelings consistent with the labels observed so far,
and chooses among them one that minimizes the resulting cutsize over the whole graph.
}
we are guaranteed to make zero mistakes on all unseen vertices. On the other
hand, if we query two nodes at random, then with constant probability both
of them will belong to the bigger star, and all the unseen labels of the smaller 
star will be mistaken.
This simple example shows that the gap between the performance of passive and 
active learning on graphs can be
made arbitrarily big.

In general, one would like to devise a strategy for placing a certain budget
of queries on the vertices of a given graph. This should be done so as to
minimize the number of mistakes made on the non-queried nodes by some
reasonable classifier like mincut. This question has been
investigated from a theoretical viewpoint by
Guillory and Bilmes~\cite{gb09}, and by Afshani et al.~\cite{acdfmsz07}. 
Our work is related to an elegant result from~\cite{gb09} which bounds the number 
of mistakes made by the mincut
classifier on the worst-case assignment of labels in terms of $\Phi/\Psi(L)$.
Here $\Phi$ is the cutsize induced by the unknown labeling, and $\Psi(L)$ is a
function of the query (or training) set $L$, which depends on the structural properties of the
(unlabeled) graph.
For instance, in the above example of the binary system,
the value of $\Psi(L)$ when the query set $L$ includes just the two centers is $1$.
This implies that for the binary system graph, Guillory and Bilmes' bound
on the mincut strategy is $\Phi$ mistakes in the worst case
(note that in the above example $\Phi \le 1$).
Since $\Psi(L)$ can be efficiently computed on any given graph and query set $L$,
the learner's task might be reduced to finding a query set $L$ that maximizes $\Psi(L)$
given a certain query budget (size of $L$). Unfortunately, no feasible general algorithm
for solving this maximization problem is known, and so one must resort to
heuristic methods ---see~\cite{gb09}.

In this work we investigate the active learning problem on graphs in the important
special case of trees. We exhibit a simple iterative algorithm which, combined
with a mincut classifier, is optimal (up to constant factors) on any given labeled tree.
This holds even if the algorithm is not 
given information on the actual cutsize $\Phi$. Our method is extremely efficient, 
requiring $\scO(n\ln Q)$ time for placing $Q$ queries in an $n$-node tree, and 
space linear in $n$. As a byproduct of our analysis, 
we show that $\Psi$ can be efficiently maximized over trees to within constant factors.
Hence the bound $\min_{L} \Phi/\Psi(L)$ can be achieved efficiently.



Another interesting question is what kind of trade-off between queries and
mistakes can be achieved if the learner is not constrained by a given
query budget. We show that a simple modification of our selection algorithm
is able to trade-off queries and mistakes in an optimal way up to
constant factors.

Finally, we prove a general lower bound for predicting the labels
of any given graph (not necessarily a tree) when the query set
is up to a constant fraction of the number of vertices. Our lower
bound establishes that the number of mistakes must then be at least a 
constant fraction of the cutsize weighted by the effective resistances.
This lower bound apparently yields a contradiction to the results of
Afshani et al.~\cite{acdfmsz07}, who constructs the query set adaptively.
This apparent contradiction is also obtained via a simple counterexample
that we detail in \secref{s:graph}.

\section{Preliminaries and basic notation}\label{s:prel}
A labeled tree $(T,\by)$ is a tree $T = (V,E)$ whose nodes $V = \{1,\dots,n\}$
are assigned binary labels $\by = (y_1,\dots,y_n)\in \{-1,+1 \}^n$.
We measure the label regularity of $(T,\by)$ by the {\em cutsize} $\Phi_T(\by)$
induced by $\by$ on $T$, i.e., $\Phi_T(\by) = \bigl|\{(i,j) \in E\,:\, y_i \neq y_j \}\bigr|$.
We consider the following \textsl{active} learning protocol: given a tree $T$ with
unknown labeling $\by$, the learner obtains all labels in a {\em query set}
$L \subseteq V$, and is then required to predict the labels of the remaining nodes
$V \setminus L$. Active learning algorithms work in two-phases:
a {\em selection} phase, where a query set of given size is constructed,
and a {\em prediction} phase, where the algorithm receives the labels
of the query set and predicts the labels of the remaining nodes.
Note that the only labels ever observed by the algorithm are those in the
query set. In particular, no labels are revealed during the prediction phase.

We measure the ability of the algorithm by the number of prediction mistakes 
made on $V \setminus L$, where it is reasonable to expect this number 
to depend on both the uknown cutsize $\Phi_T(\by)$ and the number $|L|$ of
requested labels. A slightly different prediction measure is considered
in \secref{s:q_m}.

Given a tree $T$ and a query set $L \subseteq V$,
a node $i \in V \setminus L$ is a {\bf fork node generated by} $L$
if and only if there exist three distinct nodes $i_1,i_2,i_3 \in L$ that
are connected to $i$ through edge disjoint paths. Let $\fork(L)$ be
the set of all fork nodes generated by $L$. Then $L^+$ is the query
set obtained by adding to $L$ all the generated fork nodes, i.e.,
$
L^+ \triangleq L \cup \fork(L)
$.
We say that $L \subseteq V$ is \textbf{0-forked} iff $L^+ \equiv L$.   
Note that $L^+$ is 0-forked. That is, $\fork(L^+) \equiv \emptyset$
for all $L \subseteq V$.


Given a node subset $S \subseteq V$, we use $T \setminus S$ to denote
the forest obtained by removing from the tree $T$ all nodes in $S$
and all edges incident to them. Moreover, given a second tree $T'$,
we denote by $T \setminus T'$ the forest $T \setminus V'$, where $V'$
is the set of nodes of $T'$. 
Given a query set $L \subseteq V$, a \textbf{hinge-tree} is any
connected component of $T \setminus L^+$.
We call \textbf{connection node} of a hinge-tree a node
of $L$ adjacent to any node of the hinge tree. We distinguish between 
1-hinge and 2-hinge trees. A \textbf{1-hinge-tree} has one connection
node only, whereas a \textbf{2-hinge-tree} has two (note that a hinge
tree cannot have more than two connection nodes because $L^+$ is zero-forked,
see \figref{f:hinge}).

%
\begin{figure}[ht]
\begin{center}
\scalebox{0.3}{\includegraphics{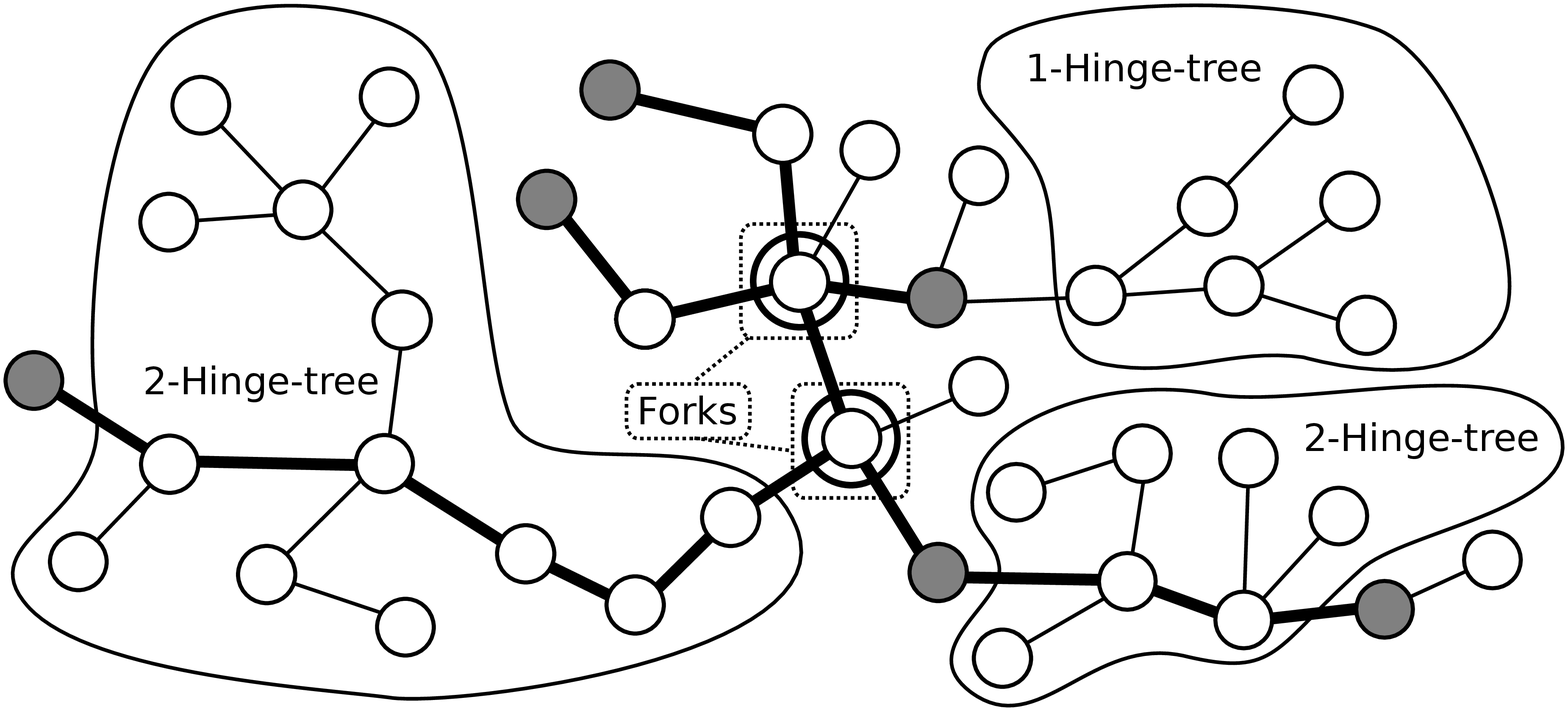}}
\end{center}
\caption{A tree $T = (V,E)$ whose nodes are shaded (the query set $L$)
or white (the set $V\setminus L$ ). The shaded nodes are also the
connection nodes of the depicted hinge trees (not all hinge trees are contoured). 
The fork nodes generated
by $L$ are denoted by double circles. The thick black edges connect the
nodes in $L$. 
\label{f:hinge}
}
\end{figure}

\section{The active learning algorithm}
\label{s:algo}
We now describe the two phases of our active learning algorithm. For the sake
of exposition, we call \sel\ the selection phase and \pred\ the prediction phase.
\sel\ returns a 0-forked query set $L^+_{\sel} \subseteq V$ of desired size.
$\pred$ takes in input the query set $L^+_{\sel}$ and the set of labels $y_i$
for all $i\in L^+_{\sel}$. Then \pred\ returns a prediction for the
labels of all remaining nodes $V\setminus L^+_{\sel}$.

\begin{figure}[ht]
\begin{center}
\scalebox{0.28}{\includegraphics{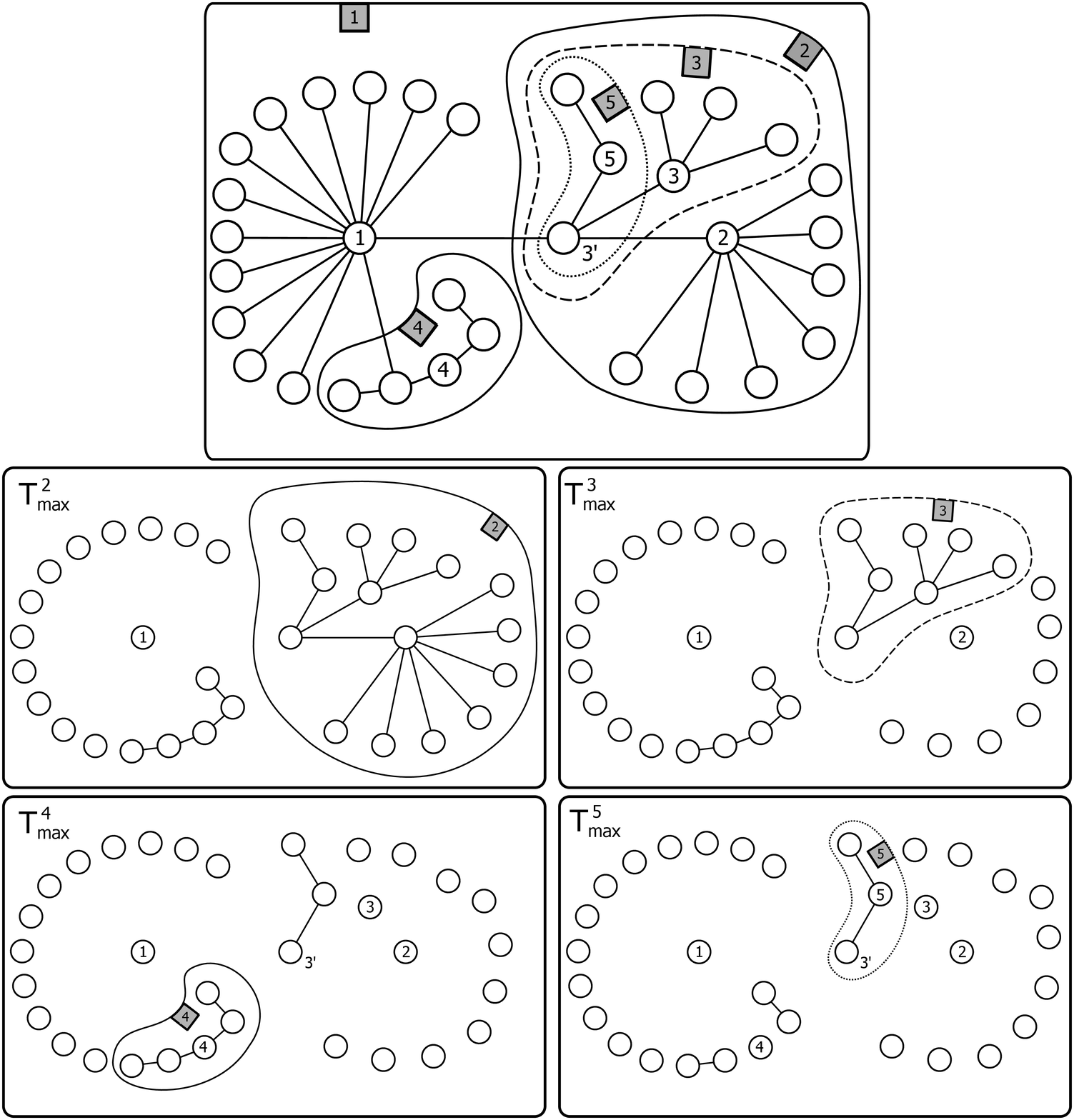}}
\end{center}
\caption{
\label{f:sel_pred}
The \sel\ algorithm at work. The upper pane shows the initial tree $T = T^1_{\max}$
(in the box tagged with ``1''), and the subsequent subtrees $T^2_{\max}$, $T^3_{\max}$, $T^4_{\max}$,
and $T^5_{\max}$. The left pane also shows the nodes selected by $\sel$ in chronological order. 
The four lower panes show the connected components of $T\setminus L_{t}$ resulting from this selection.
Observe that at the end of round~3, $\sel$ detects the generation of fork node $3'$.
This node gets stored, and is added to $L_{\sel}$ at the end of the selection process.
}
\end{figure}

In order to see the way \sel\ operates, we formally introduce
the function $\Psiinv$. This is the reciprocal of the $\Psi$
function introduced in~\cite{gb09} and mentioned in \secref{s:intro}.
\begin{definition}
Given a tree $T = (V,E)$ and a set of nodes $L \subseteq V$,
\[
\Psiinv(L) \triangleq \max_{\emptyset \not\equiv V' \subseteq V\setminus L}\ \ 
\frac{|V'|}{\bigl|\{(i,j)\in E :i \in V', j \in V\setminus V'\}\bigr|}~.
\]
\end{definition}
In words, $\Psiinv(L)$ measures the largest set of nodes not in $L$ that share
the least number of edges with nodes in $L$. From the adversary's viewpoint,
$\Psiinv(L)$ can be described as the largest return in mistakes per unit of
cutsize invested. We now move on to the description of the algorithms \sel\ and \pred.

\smallskip
The \textbf{selection algoritm} $\sel$ greedily computes a query set that
minimizes $\Psiinv$ to within constant factors.
To this end, \sel\ exploits \lemref{l:psi_upsilon} (a) (see \secref{ss:upper})
stating that, for any fixed query set $L$, the subset $V'\subseteq V$
maximizing 
$\frac{|V'|}{\bigl|\{(i,j)\in E :i \in V', j \in V\setminus V'\}\bigr|}$ 
is always included in a connected component of $T \setminus L$.
%
%
Thus \sel\ places
its queries in order to end up with a query set $L^+_{\sel}$ such that
the largest component of $T \setminus L^+_{\sel}$ is as small as possible.

$\sel$ operates as follows. 
Let $L_t \subseteq L$ be the set including the first $t$ nodes chosen by $\sel$,
$T^t_{\max}$ be the largest connected component of $T \setminus L_{t-1}$, 
and $\sigma(T',i)$ be the size (number of nodes) of the largest component of the 
forest $T' \setminus \{i\}$, where $T'$ is any tree. 
At each step $t = 1,2,\dots$, \sel\ simply picks the node $i_t \in T^{t}_{\max}$
that minimizes $\sigma(T^{t}_{\max},i)$ over $i$
and sets $L_t = L_{t-1} \cup \{i_t\}$. During this iterative construction,
$\sel$ also maintains a set containing all fork nodes
generated in each step by adding nodes $i_t$ to the sets $L_{t-1}$.\footnote{
In Section~\ref{s:impl} we will see that during each step $L_{t-1} \rightarrow L_t$ 
at most a single new fork node may be generated.
}
After the desired number of queries is reached (also counting the queries
that would be caused by the stored fork nodes), \sel\ has terminated the construction
of the query set $L_{\sel}$. The final query set $L^+_{\sel}$, obtained
by adding all stored fork nodes to $L_{\sel}$, is then returned.

\smallskip
The \textbf{Prediction Algorithm \pred}
receives in input the labeled nodes of the 0-forked query set $L^+_{\sel}$
and computes a mincut assignment.
Since each component of $T \setminus L^+_{\sel}$ is either a 1-hinge-tree or a
2-hinge-tree, \pred\ is simple to describe and is also very efficient.
The algorithm predicts all the nodes of hinge-tree $\scT$ using the same label
$\yhat_{\scT}$. This label is chosen according to the following two cases:
\begin{enumerate}
\item If $\scT$ is a 1-hinge-tree, then $\yhat_{\scT}$ is set to the label of its 
unique connection node;
\item If $\scT$ is a 2-hinge-tree and the labels of its two connection nodes are 
equal, then $\yhat_{\scT}$ is set to the label of its connection nodes, otherwise
$\yhat_{\scT}$ is set as the label of the closer connection node (ties are
broken arbitrarily).
\end{enumerate}
In Section~\ref{s:impl} we show  that \sel\ requires overall $\scO(|V|\log Q)$
time and $\scO(|V|)$ memory space for selecting $Q$ query nodes.
Also, we will see that the total running time taken by \pred\ for predicting
all nodes in $V \setminus L$ is linear in $|V|$.

\newcommand{\argmin}[1]{\underset{#1}{\mathrm{argmin}}}
\newcommand{\argmax}[1]{\underset{#1}{\mathrm{argmax}}}
\renewcommand{\inbd}[1]{\partial{#1}}

\section{Analysis}
\label{s:algo_an}
For a given tree $T$, we denote by $m_A(L,\by)$ the number of prediction
mistakes that algorithm $A$ makes on the labeled tree $(T,\by)$ when
given the query set $L$.
Introduce the function
\[
    m_A(L,K) = \max_{\by \,:\, \Phi_T(\by) \le K} m_A(L,\by)
\]
denoting the number of prediction mistakes
made by $A$ with query set $L$
on all labeled trees with cutsize bounded by $K$. 
We will also find it useful to deal with the ``lower bound''
function $\lb(L,K)$. This is the maximum expected number of 
mistakes that any prediction algorithm $A$ can be forced to make 
on the labeled tree $(T,\by)$ when the query set is $L$ and the 
cutsize is not larger than $K$. 

We show that the number of mistakes made by \pred\ on any labeled tree
when using the query set $L^+_{\sel}$ satisfies
\[
    m_{\pred}(L^+_{\sel},K) \le 10\,\lb(L,K)
\]
for all query sets $L\subseteq V$ of size up to $\tfrac{1}{8}|L^+_{\sel}|$.
Though neither \sel\ nor \pred\ do know the actual cutsize of the
labeled tree $(T,\by)$, the combined use of these procedures is
competitive against any algorithm that knows the cutsize budget $K$ beforehand.


While this result implies the optimality (up to constant factors) of our
algorithm, it does not relate the mistake bound to the cutsize, which
is a clearly interpretable measure of the label regularity.
In order to address this issue, we show that our algorithm also satisfies
the bound
\[
    m_{\pred}(L^+_{\sel},\by) \le 4\,\Psiinv(L)\,\Phi_T(\by)
\]
for all query sets $L\subseteq V$ of size up to $\tfrac{1}{8}|L^+_{\sel}|$.
The proof of these results needs a number of preliminary lemmas.
\begin{lemma}
\label{l:sigma}
For any tree $T = (V,E)$ it holds that
${\displaystyle \min_{v \in V}\sigma(T,v) \le \tfrac{1}{2}|V|}$.
\end{lemma}
\begin{proof}
Let $i \in \mathrm{argmin}_{v \in V}\sigma(T,v)$. For the sake of contradiction, 
assume there exists a component $T_i = (V_i,E_i)$ of $T \setminus \{i\}$ such 
that $|V_i| > |V|/2$. Let $s$ be the sum of the sizes all other components.
Since $|V_i|+s = |V|-1$, we know that $s \le |V|/2-1$.
Now let $j$ be the node adjacent to $i$ which belongs to $V_i$ and
$T_j = (V_j,E_j)$ be the largest component of $T \setminus \{j\}$.
There are only two cases to consider: either $V_j \subset V_i$ or $V_j \cap V_i \equiv \emptyset$.
In the first case, $|V_j| < |V_i|$.
In the second case, $V_j \subseteq \{i\} \cup \bigl(T \setminus V_i\bigr)$,
which implies $|V_j| \le 1 + s \le |V|/2 < |V_i|$. In both cases,
$i \not\in\mathrm{argmin}_{v \in V}\sigma(T,v)$, which provides the desired contradiction.
\end{proof}
\begin{lemma}\label{l:deforked}
For all subsets $L \subset V$ of the nodes of a tree $T = (V,E)$ we have
$
\big|L^{+}\big| \le 2|L|
$.
\end{lemma}
\begin{proof}
Pick an arbitrary node of $T$ and perform a depth-first visit
of all nodes in $T$. This visit induces an ordering $\scT_1,\scT_2,\dots$ of
the connected components in $T\setminus L$ based on the order of the nodes visited first
in each component.
Now let $\scT'_1,\scT'_2,\dots$ be such that each $\scT'_i$ is a component of
$\scT_i$ extended to include all nodes of $L$ adjacent to nodes in $\scT_i$.
Then the ordering implies that, for $i \ge 2$, $\scT'_i$ shares exactly one node
(which must be a leaf) with all previously visited trees.
Since in any tree the number of nodes
of degree larger than two must be strictly smaller than the number of leaves, we have
$|\fork(\scT'_i)| < |\Lambda_i|$ where, 
with slight abuse of notation,  we denote by $\fork(\scT'_i)$ the set of all fork nodes 
in subtree $\scT'_i$. Also, we let $\Lambda_i$ be the set of leaves of $\scT'_i$.
This implies that, for $i=1,2,\dots$, each fork node in $\fork(\scT'_i)$ can be injectively
associated with one of the $|\Lambda_i|-1$ leaves of $\scT'_i$ that are not shared
with any of the previously visited trees. 
Since $|\fork(L)|$ is equal to the sum of $|\fork(\scT_i)|$ over all indices $i$,
this implies that $|\fork(L)| \le |L|$.
\end{proof}

\newcommand{\Tnol}{T_{\mathrm{nol}}}
\newcommand{\Lambdaa}{\Lambda_{\mathrm{add}}}
\begin{lemma}\label{l:split}
Let $L_{t-1} \subseteq L_{\sel}$ be the set of the first $t-1$ nodes chosen by \sel.
Given any tree $T = (V,E)$, the largest subtree of $T \setminus L_{t-1}$
contains no more than $\tfrac{2}{t}|V|$ nodes.
\end{lemma}
\begin{proof}
Recall that $i_s$ denotes the $s$-th node selected by \sel\ during the incremental construction of
the query set $L_{\sel}$, and that
$T^s_{\max}$ is the largest component of $T \setminus L_{s-1}$.
The first $t$ steps of the recursive splitting procedure performed
by \sel\ can be associated with a splitting tree $T'$ defined
in the following way.
The internal nodes of $T'$ are $T^s_{\max}$, for $s \ge 1$.
The children of $T^s_{\max}$ are the connected components
of $T^s_{\max} \setminus \{i_s\}$, i.e., the subtrees of
$T^s_{\max}$ created by the selection of $i_s$. 
Hence, each leaf of $T'$ is bijectively associated with a tree
in $T \setminus L_t$.

Let $\Tnol'$ be the tree obtained from $T'$ by deleting all leaves.
Each node of $\Tnol'$ is one of the $t$ subtrees split by \sel\ during
the construction of $L_t$. As $T^t_{\max}$ is split by $i_t$, it is
a leaf in $\Tnol'$.
We now add a second child to each internal node $s$ of $\Tnol'$ having
a single child. This second child of $s$ is obtained by merging all the
subtrees belonging to leaves of $T'$ that are also children of $s$.
Let $T''$ be the resulting tree.


We now compare the cardinality of $T^t_{\max}$ to that of the subtrees
associated with the leaves of $T''$. 
Let $\Lambda$ be the set of all leaves of $T''$ and
$\Lambdaa = T'' \setminus \Tnol' \subset \Lambda$ be
the set of all leaves added to $\Tnol'$ to obtain $T''$. 
First of all, note that $|T^t_{\max}|$ is not larger than the number
of nodes in any leaf of $\Tnol'$.
This is because the selection rule of \sel\ ensures that $T^t_{\max}$
cannot be larger than any subtree associated with a leaf in $\Tnol'$, since
it contains no node selected before time $t$.
In what follows, we write $|s|$ to denote the size of the forest or
subtree associated with a node $s$ of $T''$.
We now prove the following claim:

\smallskip\noindent
\textbf{Claim.} For all $\ell\in\Lambda$, $|T^t_{\max}| \le |\ell|$,
and for all $\ell\in\Lambdaa$, $|T^t_{\max}|-1 \le |\ell|$.

\smallskip\noindent
\textsl{Proof of Claim.}
The first part just follows from the observation that any $\ell\in\Lambda$
was split by \sel\ before time $t$. In order to prove the second part,
pick a leaf $\ell \in \Lambdaa$. Let $\ell'$ be its unique sibling in 
$T''$ and let $p$ be the parent of $\ell$ and $\ell'$, also in $T''$.
\lemref{l:sigma} applied to the subtree $p$ implies $|\ell'| \le \tfrac{1}{2}|p|$.
Moreover, since $|\ell|+|\ell'| = |p|-1$, we obtain $|\ell| + 1 \ge \tfrac{1}{2}|p| \ge
|\ell'| \ge |T^t_{\max}|$, the last inequality using the first part of the claim.
This implies $|T^t_{\max}| - 1 \le |\ell|$, and the claim is proven.

\smallskip\noindent
Let now $N(\Lambda)$ be the number of nodes in subtrees and forests
associated with the leaves of $T''$. With each internal node of $T''$
we can associate a node of $L_{\sel}$ which does not belong
to any leaf in $\Lambda$. Moreover, the number $|T''\setminus\Lambda|$
of internal nodes in $T''$ is bigger than the number $|\Lambdaa|$ of
internal nodes of $\Tnol'$ to which a child has been added.
Since these subtrees and forests are all distinct, we obtain
$N(\Lambda) + |T''\setminus\Lambda| < N(\Lambda) + |\Lambdaa| \le |V|$.
Hence, using the above claim we can write
$N(\Lambda) \ge \bigl(|\Lambda| - |\Lambdaa|\bigr)|T^t_{\max}|
+ |\Lambdaa|\bigl(|T^t_{\max}|-1\bigr)$, which implies 
$|T^t_{\max}| \le \bigl(N(\Lambda)+|\Lambdaa|\bigr)/|\Lambda|
\le |V|/|\Lambda|$.  Since each internal node of $T''$ has at least 
two children, we have that $|\Lambda| \ge |T''|/2 \ge |\Tnol'|/2 = t/2$.
Hence, we can conclude that $|T^t_{\max}| \le 2|V|/t$.
\end{proof}

\subsection{Lower bounds}
We now state and prove a lower bound on the number of mistakes that any prediction
algorithm (even knowing the cutsize budget $K$) makes on any given tree, when the query
set $L$ is 0-forked.
The bound depends on the following quantity:
Given a tree $T(V,E)$, a node subset $L \subseteq V$ and an integer $K$, the
\textbf{component function} $\Upsilon(L,K)$ is the sum of the sizes
of the $K$ largest components of $T \setminus L$, or~~ $|V\setminus L|$~~ if $T \setminus L$
has less than $K$ components.
\begin{theorem}\label{th:lb}
For all trees $T = (V,E)$, for all 0-forked subsets $L^+ \subseteq V$,
and for all cutsize budgets $K = 0,1,\dots,|V|-1$, we have that
$
    \lb(L^+,K) \ge \tfrac{1}{2}\Upsilon(L^+,K)
$.
\end{theorem}
\begin{proof}
We describe an adversarial strategy causing any algorithm to make at
least $\Upsilon(L^+,K)/2$ mistakes even when the cutsize budget $K$
is known beforehand. Since $L^+$ is 0-forked, each component of $T \setminus
L^+$ 
is a hinge-tree. Let $F_{\max}$ be the set of the 
$K$
largest hinge-trees of $T \setminus L^+$, and $E(\scT)$ be the set of all
edges in $E$ incident to at least one node of a hinge-tree $\scT$.
The adversary creates at most one $\phi$-edge\footnote
{
A $\phi$-edge $(i,j)$ is one where $y_i \neq y_j$. 
} 
in each edge set $E(\scT_1)$
for all 1-hinge-trees $\scT_1 \in F_{\max}$, exactly one
$\phi$-edge in each edge set $E(\scT_2)$ for all 2-hinge-trees
$\scT_2 \in F_{\max}$, and no $\phi$-edges in the edge set
$E(\scT)$ of any remaining hinge-tree $\scT \not\in F_{\max}$.
This is done as follows.
By performing a depth-first visit of $T$, the adversary can always assign
disagreeing labels to the two connection nodes of each 2-hinge-tree
in $F_{\max}$, and agreeing labels to the two connection
nodes of each 2-hinge-tree not in $F_{\max}$. Then,
for each hinge-tree $\scT\in F_{\max}$, the adversary assigns
a unique random label to all nodes of $\scT$, forcing $|\scT|/2$ mistakes
in expectation. The labels of the remaining hinge-trees not in
$F_{\max}$ are chosen in agreement with their connection nodes.
\end{proof}
%
%
\begin{remark}
Note that \thmref{th:lb} holds for \textsl{all} query sets, not only those that are 0-forked,
since any adversarial strategy for a query set $L^+$ can force at least the same mistakes on the
subset $L \subseteq L^+$.
Note also that it is not difficult to modify the adversarial strategy described in the proof of
\thmref{th:lb} in order to deal with algorithms that are allowed to adaptively choose the query
nodes in $L$ depending on the labels of the previously selected nodes.
The adversary simply assigns the same label to each node in the query set and then forces, with the same
method described in the proof, $\tfrac{1}{2}\Upsilon\bigl(L^+,\tfrac{K}{2}\bigr)$ mistakes in expectation
on the $\tfrac{K}{2}$ largest hinge-trees. Thus there are at most two $\phi$-edges in each edge set
$E(\scT)$ for all hinge-trees $\scT$, yielding at most $K$ $\phi$-edges in total.
The resulting (slightly weaker) bound is $\lb(L^+,K) \ge \tfrac{1}{2}\Upsilon\bigl(L^+,\tfrac{K}{2}\bigr)$.
\thmref{t:together} and \corref{th:bto_f} can also be easily rewritten in order to extend the results in
this direction.
\end{remark}

\subsection{Upper bounds}\label{ss:upper}
We now bound the total number of mistakes that \pred\ makes on any labeled tree
when the queries are decided by \sel. We use Lemma~\ref{l:sigma}
and~\ref{l:deforked}, together with the two lemmas below, to prove that
$m_{\pred}(L_{\sel}^+,K) \le 10\,\lb(L,K)$ for all cutsize budgets $K$
and for all node subset $L \subseteq V$ such that $|L| \le \tfrac{1}{8}|L_{\sel}^+|$.  
%
\begin{lemma}\label{l:ub_bto}
For all labeled trees $(T,\by)$ and for all 0-forked query sets $L^+ \subseteq V$,
the number of mistakes made by \pred\ satisfies 
$
    m_{\pred}(L^+,\by) \le \Upsilon\bigl(L^+,\Phi_T(\by)\bigr)
$.
\end{lemma}
\begin{proof}
As in the proof of Theorem \ref{th:lb}, we first observe that
each component of $T \setminus L^+$ is a hinge-tree.
Let $E(\scT)$ be the set of all edges in
$E$ incident to nodes of a hinge-tree $\scT$, and 
$F_{\phi}$ be the set of hinge-trees such that, for all
$\scT \in F_{\phi}$, at least one edge of $E(\scT)$ is a $\phi$-edge.
Since $E(\scT) \cap E(\scT') \equiv \emptyset$ for all
$\scT,\scT' \in T \setminus L^+$, we have that $|F_{\phi}| \le \Phi_T(\by)$.
Moreover, since for any $\scT\not\in F_{\phi}$ there are no $\phi$-edges in $E(\scT)$,
the nodes of $\scT$ must be labeled as its connections nodes.
This, together with the prediction rule of \pred, implies that \pred\ makes
no mistakes over any of the hinge-trees $\scT\not\in F_{\phi}$.
Hence, the number of mistakes made by \pred\ is bounded by the sum of the sizes
of all hinge-trees $\scT \in F_{\phi}$, which (by definition of $\Upsilon$) is
bounded by $\Upsilon\bigl(L^+,\Phi_T(\by)\bigr)$.
\end{proof}

\begin{figure}[h!]
\begin{center}
\scalebox{0.28}{\includegraphics{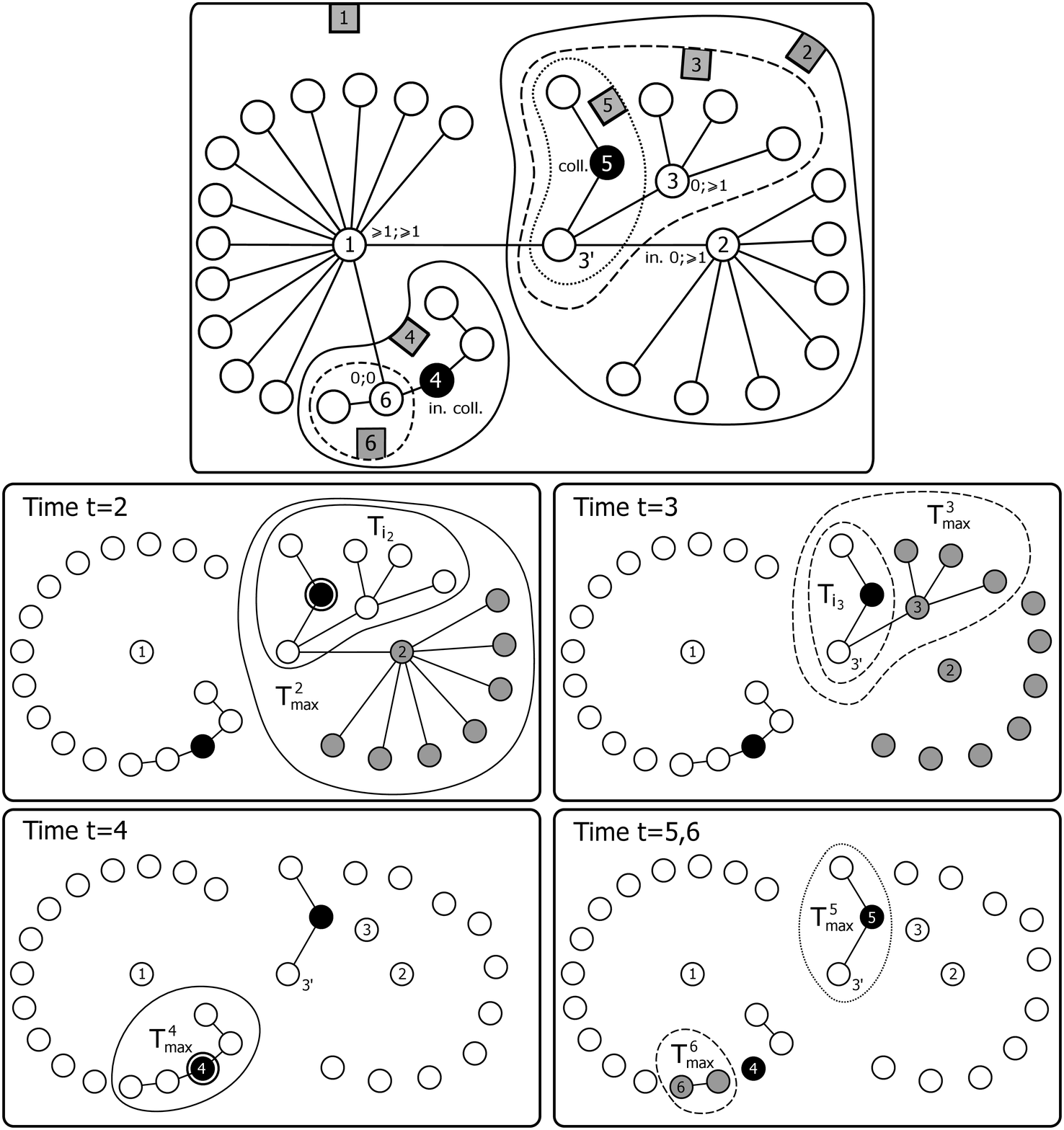}}
\end{center}
\caption{
The upper pane illustrates the different kinds of nodes chosen by \sel. 
Numbers in the square tags indicate the first six subtrees $T^t_{\max}$,
and their associated nodes $i_t$, selected by \sel. Node $i_1$ is a $[\ge1;\ge1]$-node, 
$i_2$ is an initial $[0;\ge1]$-node, $i_3$ is a (noninitial) $[0;\ge1]$-node, 
$i_4$ is an initial collision node, $i_5$ is a (noninitial) collision node, 
and $i_6$ is a $[0;0]$-node. As in \figref{f:sel_pred}, we denote by $3'$ the fork node 
generated by the inclusion of $i_3$ into $L_{\sel}$. 
Note that node $i_6$ may be chosen arbitrarily among the four nodes in 
$T^4_{\max} \setminus i_4$. The two black nodes are the set of nodes we are competing against, 
i.e., the nodes in the query set $L$.
Forest $T \setminus L$ is made up of one large subtree and two small subtrees.  
In the lower panes we illustrate some steps of the proof of \lemref{l:upsilon}, 
with reference to the upper pane.
\underline{Time $t=2$:} Trees $T^2_{\max}$ and $T_{i_2}$ are shown. As explained in the proof,
$|T_{i_2}| \le |T^2_{\max} \setminus T_{i_2}|$. The circled black node is captured by $i_2$. 
The nodes of tree $T^2_{\max} \setminus T_{i_2}$ are shaded, and can be used for mapping any 
$\zeta$-component through $\mu_2$. 
\underline{Time $t=3$:} Trees $T^3_{\max}$ and $T_{i_3}$ are shown. 
Again, one can easily verify that $|T_{i_3}| \le |T^3_{\max} \setminus T_{i_3}|$. 
As before, the nodes of $T^3_{\max} \setminus T_{i_3}$ are shaded, and can be used for mapping 
any $\zeta$-component via $\mu_2$. 
The reader can see that, according to the injectivity of $\mu_2$, 
these grey nodes are well separated from the ones in $T^2_{\max} \setminus T_{i_2}$. 
\underline{Time $t=4$:} $T^4_{\max}$ and the initial collision node $i_4$ are depicted.
The latter is enclosed in a circled black node since it captures itself. 
\underline{Time $t=5,6$:} We depicted trees $T^5_{\max}$ and $T^6_{\max}$, together with nodes 
$i_5$ and $i_6$. Node $i_5$ is a collision node, which is not initial 
since it was already captured by the $[0;\ge1]$-node $i_2$. Node $i_6$ is a $[0;0]$ node, 
so that the whole tree $T^6_{\max}$ is completely included in a component 
(the largest, in this case) of $T \setminus L$. Tree $T^6_{\max}$ can be used for mapping 
via $\mu_3$ any $\zeta$-component. 
The resulting forest $T \setminus L_6$ includes several single-node trees and one two-node tree. 
If $i_6$ is the last node selected by $L_{\sel}$, then each component of $T \setminus L_6$ can 
be exploited by mapping $\mu_1$, since in this specific case none of these components contains nodes 
of $L$, i.e., there are no $\zeta$-components left.
\label{f:3maps}
}
\end{figure}

The next lemma, whose proof is a bit involved, 
provides the relevant properties of the component function $\Upsilon(\cdot,\cdot)$.
Figure~\ref{f:3maps} helps visualizing the main ingredients of the proof.
\begin{lemma}\label{l:upsilon}
Given a tree $T = (V,E)$, for all node subsets $L \subseteq V$
such that $|L| \le \tfrac{1}{2}|L_{\sel}|$ and for all integers $k$, we have:
\enspace (a) $\Upsilon(L_{\sel}, k) \le 5\Upsilon(L, k)$;
\enspace (b) $\Upsilon(L_{\sel},1) \le \Upsilon(L,1)$. 
\end{lemma} 
\begin{proof}
We prove part~(a) by constructing, via $\sel$, 
three bijective mappings $\mu_1, \mu_2, \mu_3 : \scP_{\sel} \to \scP_L$,
where $\scP_{\sel}$ is a suitable partition of $T \setminus L_{\sel}$, 
$\scP_L$ is a subset of $2^V$ such that any $S \in \scP_L$ is all contained
in a single connected component of $T \setminus L$, 
and the union of the domains of the three mappings covers the whole set 
$T \setminus L_{\sel}$.
The mappings $\mu_1$, $\mu_2$ and $\mu_3$ are shown to satisfy, for all forests\footnote
{
In this proof, $|\mu(A)|$ denotes the number of nodes in the set (of nodes) $\mu(A)$.
Also, with a slight abuse of notation, for all forests $F\in \scP_{\sel}$, 
we denote by $|F|$ the sum of the number of nodes in all trees of $F$.
Finally, whenever $F \in \scP_{\sel}$ contains a single tree, we refer to $F$
as it were a tree, rather than a (singleton) forest containing only one tree.
} 
$F \in \scP_{\sel}$,
\[
    |F| \le |\mu_1(F)|,\qquad |F| \le 2|\mu_2(F)|,\qquad |F| \le 2|\mu_3(F)|~~.
\]
%
%
Since each $S \in \scP_L$ is all contained in a connected component of 
$T \setminus L$, this we will enable us to conclude that, for each
tree $T' \in T \setminus L$, the forest of all trees $T \setminus L_{\sel}$
mapped (via any of these mappings) to any node subset of $T'$ has at most
five times the number of nodes of $T'$. This would prove the statement in~(a).


The construction of these mappings requires some auxiliary definitions.
We call $\zeta$-component each connected component of $T \setminus L_{\sel}$
containing at least one node of $L$. Let
$i_t$ be the $t$-th node selected by \sel\ during the incremental construction 
of the query set $L_{\sel}$. We distinguish between four kinds of nodes chosen by \sel
---see \figref{f:3maps} for an example.

Node $i_t$ is:
\begin{enumerate}
\item A \textit{collision node} if it belongs to $L_{\sel} \cap L$;
\item a \textit{$[0;0]$-node} if, at time $t$, the tree $T^t_{\max}$ does not 
contain any node of $L$;
\item a \textit{$[0;\ge1]$-node} if, at time $t$, the tree $T^t_{\max}$ contains 
$k\ge 1$ nodes $j_1,\ldots,j_k \in L$ all belonging to the same connected 
component of $T^t_{\max} \setminus \{i_t\}$;
\item a \textit{$[\ge1;\ge1]$-node} if $i_t \not\in L$ and, at time $t$, the tree $T^t_{\max}$ 
contains $k\ge 2$ nodes $j_1,\ldots,j_k \in L$, which do not 
belong to the same connected component of $T^t_{\max} \setminus \{i_t\}$.
\end{enumerate}
We now turn to building the three mappings. 

$\mu_1$ simply maps each tree $T' \in T \setminus L_{\sel}$ that is {\em not} a $\zeta$-component 
to the node set of $T'$ itself. This immediately
implies $|F| \le |\mu_1(F)|$ for all forests $F$ (which are actually single trees) 
in the domain of $\mu_1$.
Mappings $\mu_2$ and $\mu_3$ deal with the $\zeta$-components of $T\setminus L_{\sel}$.
Let $Z$ be the set of all such $\zeta$-components, and denote by $V_{0;0}$, $V_{0;1}$,
and $V_{1;1}$ the set of all $[0;0]$-nodes, $[0;\ge1]$-nodes, and $[\ge1;\ge1]$-nodes, 
respectively.
Observe that $|V_{1;1}| < |L|$. Combined with
the assumption $|L_{\sel}| \ge 2|L|$, this implies that
$|V_{0;0}| + |V_{0;1}|$ plus the total number of collision nodes must be larger than $|L|$;
as a consequence, $|V_{0;0}| + |V_{0;1}| > |Z|$. 
Each node $i_t \in V_{0;1}$ chosen by $\sel$ splits the tree $T^t_{\max}$ into 
one component $T_{i_t}$ containing at least one node of $L$ and one or more components
all contained in a single tree $T'_{i_t}$ of $T \setminus L$. Now mapping $\mu_2$ can be
constructed incrementally in the following way. 
For each $[0;\ge1]$-node selected by $\sel$ at time $t$, $\mu_2$ sequentially
maps any $\zeta$-component generated 
to the set of nodes in $T^t_{\max} \setminus T_{i_t}$, the latter being just
a subset of a component of $T\setminus L$. A future time step $t' > t$
might feature the selection of a new $[0;\ge1]$-node within $T_{i_t}$, but
mapping $\mu_2$ would cover a different subset of such component of $T\setminus L$.
Now, applying \lemref{l:sigma} to tree $T^t_{\max}$, we can see that 
$|T^t_{\max} \setminus T_{i_t}| \ge |T^t_{\max}|/2$. 
Since the selection rule of \sel\ guarantees that the number of nodes in 
$T^t_{\max}$ is larger than the number of nodes of any $\zeta$-component, 
we have $|F| \le 2|\mu_2(F)|$, for any $\zeta$-component $F$ considered in 
the construction of $\mu_2$.

Mapping $\mu_3$ maps all the remaining $\zeta$-components that
are not mapped through $\mu_2$. Let $\sim$ be an equivalence relation 
over $V_{0;0}$ defined as follows: $i \sim j$ iff $i$ is connected to $j$ 
by a path containing only $[0;0]$-nodes
and nodes in $V \setminus (L_{\sel} \cup L)$.
Let $i_{t_1}, i_{t_2}, \ldots, i_{t_k}$ be the sequence of nodes of any given equivalence
class $[C]_{\sim}$, sorted according to $\sel$'s chronological selection.
\lemref{l:split} applied to tree $T^{t_1}_{\max}$ shows that  
$|T^{t_k}_{\max}| \le 2|T^{t_1}_{\max}|/k$. Moreover,
the selection rule of \sel\ guarantees that the number of nodes of $T^{t_k}_{\max}$ 
cannot be smaller than the number of nodes of any $\zeta$-component. 
Hence, for each equivalence class $[C]_{\sim}$ containing $k$ nodes of type
$[0;0]$, we map through $\mu_3$ a set $F_{\zeta}$ of $k$ arbitrarily chosen 
$\zeta$-components to $T^{t_1}_{\max}$. Since the size of each $\zeta$-component
is $\leq |T^{t_k}_{\max}|$, we can write 
$|F_{\zeta}| \le k|T^{t_k}_{\max}| \le 2|T^{t_1}_{\max}|$,
which implies $|F_{\zeta}| \le 2|\mu_3(F_{\zeta})|$ for all $F_{\zeta}$ in
the domain of $\mu_3$.
Finally, observe that the number of $\zeta$-components that are not mapped through $\mu_2$ 
cannot be larger than $|V_{0;0}|$, thus the union of mappings $\mu_2$ and $\mu_3$ do actually 
map all $\zeta$-components. 
This, in turn, implies that the union of the domains of the three mappings covers the whole set 
$T \setminus L_{\sel}$, thereby concluding the proof of part (a).

The proof of (b) is built on the definition of collision nodes, $[0;0]$-nodes,
$[0;\ge1]$-nodes and $[\ge1;\ge1]$-nodes given in part (a). 
Let $L_t \subseteq L_{\sel}$ be the set of the first $t$ nodes
chosed by \sel. 
Here, we make a further distinction within the collision and $[0;\ge1]$-nodes. 
We say that during the selection of node $i_t \in V_{0;1}$, 
the nodes in $L \cap T^t_{\max}$ are {\em captured} by $i_t$. This notion of 
capture extends to collision nodes by saying that a collision node $i_t \in L \cap L_{\sel}$ 
just {\em captures itself}.
We say that $i_t$ is an \textit{initial} $[0;\ge1]$-node (resp., {\em initial} collision node)
if $i_t$ is a $[0;\ge1]$-node (resp., collision node) such that the whole set of nodes in $L$
captured by $i_t$ contains no nodes captured so far. See Figure \ref{f:3maps} for reference.
%
%
The simple observation leading to the proof of part (b) is the following. If
$i_t$ is a $[0;0]$-node, then $T^t_{\max}$ cannot be larger than
the component of $T \setminus L$ that contains $T^t_{\max}$, which in turn
cannot be larger than $\Upsilon(L,1)$. This would already imply
$\Upsilon(L_{t-1},1) \le \Upsilon(L,1)$.
Let now $i_t$ be an initial $[0;\ge1]$-node and $T_{i_t}$ be the 
unique component of $T^t_{\max} \setminus \{i_t\}$ containing
one or more nodes of $L$. Applying \lemref{l:sigma} to tree $T^t_{\max}$ we can 
see that $|T_{i_t}|$ cannot be larger than $|T^t_{\max} \setminus T_{i_t}|$,
which in turn cannot be larger than $\Upsilon(L,1)$.
If at time $t' > t$ the procedure $\sel$ selects $i_{t'} \in T_{i_t}$ then
$|T^{t'}_{\max}| \leq |T_{i_t}| \leq \Upsilon(L,1)$.
%
%
Hence, the maximum integer $q$ such that $\Upsilon(L_{q},1) > \Upsilon(L,1)$
is bounded by the number of $[\ge1;\ge1]$-nodes plus the number of initial $[0;\ge1]$-nodes plus 
the number of initial collision nodes.
We now bound this sum as follows.
The number of $[\ge1;\ge1]$-nodes is clearly bounded by $|L|-1$. 
Also, any initial $[0;\ge1]$-node or initial collision node selected by $\sel$ 
captures at least a new node in $L$, thereby implying that the total number of 
initial $[0;\ge1]$-node or initial collision node must be $\leq |L|$.
After $q = 2|L|-1$ rounds, we are sure that the size of the largest tree of 
$T^{q}_{\max}$ is not larger than the size of the largest component of 
$T\setminus L$, i.e., $\Upsilon(L,1)$~.
\end{proof}    

%
%

We now put the above lemmas together to prove our main result 
concerning the number of mistakes made by \pred\ on the query set chosen by \sel.
\begin{theorem}\label{t:together}
For all trees $T$ and all cutsize budgets $K$,
the number of mistakes made by \pred\ on the query set $L_{\sel}^+$ satisfies 
\[
    m_{\pred}(L_{\sel}^+, K)
\le 
    \min_{L \subseteq V \,:\, |L| \le \tfrac{1}{8}|L_{\sel}^+|}
    10\,\lb\bigl(L,K\bigr)~.
\]
\end{theorem}  
\begin{proof}
Pick any $L \subseteq V$ such that $|L| \le \tfrac{1}{8}|L_{\sel}^+|$.
Then
\begin{align*}
    m_{\pred}(L_{\sel}^+,K)
\!\!\overset{\text{(Lem.~\ref{l:ub_bto})}}{\le}\!\!
    \Upsilon(L_{\sel}^+,K) 
\overset{\text{(A)}}{\le}
    \Upsilon(L_{\sel},K)
\!\!\overset{\text{(Lem.~\ref{l:upsilon} (a))}}{\le}\!\!
    5\Upsilon(L^{+},K)
\!\!\overset{\text{(Thm.~\ref{th:lb})}}{\le}\!\!
    10\,\lb(L^{+},K)
\overset{\text{(B)}}{\le}
    10\,\lb(L,K)~.
\end{align*}
Inequality~(A) holds because $L_{\sel} \subseteq L_{\sel}^+$, and thus $T \setminus L_{\sel}^+$
has connected components of smaller size than $L_{\sel}$. In order to apply \lemref{l:upsilon} (a),
we need the condition $|L^+| \le \tfrac{1}{2}|L_{\sel}|$. 
This condition is seen to hold after combining \lemref{l:deforked} with our
assumptions: $|L^+| \le 2|L| \le \tfrac{1}{4}|L_{\sel}^+| \le \tfrac{1}{2}|L_{\sel}|$.
Finally, inequality~(B) holds because any adversarial strategy using query set $L$ can
also be used with the larger query set $L^+ \supseteq L$.
\end{proof}

Note also that \thmref{th:lb} and \lemref{l:ub_bto} imply the following
statement about the optimality of \pred\ over 0-forked query sets.
\begin{corollary}
\label{th:bto_f}
For all trees $T$, for all cutsize budgets $K$,
and for all 0-forked query sets $L^+ \subseteq V$,
the number of mistakes made by \pred\ satisfies
$
    m_{\pred}(L^+,K) \le 2\lb\bigl(L^+,K\bigr)
$.
\end{corollary}
In the rest of this section we derive a more intepretable bound
on $m_{\pred}(L^+, \by)$ based on the function $\Psiinv$ introduced
in~\cite{gb09}. To this end, we prove that $L_{\sel}$ minimizes
$\Psiinv$ up to constant factors, and thus is an optimal query
set according to the analysis of~\cite{gb09}.

For any subset $V' \subseteq V$,
let $\Gamma(V',V \setminus V')$ be the number of edges between nodes
of $V'$ and nodes of $V \setminus V'$. Using this notation, we can
write
\[
    \Psiinv(L) = \max_{\emptyset \not\equiv V' \subseteq V\setminus L}
    \frac{|V'|}{\Gamma(V',V \setminus V')}~.
\]
\begin{lemma}\label{l:psi_upsilon}
For any tree $T = (V,E)$ and any $L \subseteq V$ the following holds.
\begin{itemize}
\item [(a)] A maximizer of $\frac{|V'|}{\Gamma(V',V \setminus V')}$ 
exists which is included in the node set of a single
component of $T\setminus L$;
\item [(b)] $\Psiinv(L) \le \Upsilon(L,1)$.
\end{itemize}
\end{lemma} 
\begin{proof}
Let $V'_{\max}$ be any maximizer of $\frac{|V'|}{\Gamma(V',V \setminus V')}$.
For the sake of contradiction, assume that the nodes of $V'_{\max}$ belong to
$k \ge 2$ components $\scT_1,\scT_2,\ldots,\scT_k\in T\setminus L$.
Let $V'_i \subset V'_{\max}$ be the subset of nodes included in the node set
of $\scT_i$, for $i=1,\dots,k$. Then $|V'| = \sum_{i \le k} |V'_i|$ and
$\Gamma(V',V \setminus V') = \sum_{i \le k}\Gamma(V'_i,V \setminus V'_i)$.
Now let $i^* = \mathrm{argmax}_{i \le k} |V'_i|/\Gamma(V'_i,V \setminus V'_i)$.
Since $\bigl(\sum_i a_i\bigr)\big/\bigl(\sum_i b_i\bigr) \le \max_i a_i/b_i$
for all $a_i,b_i \ge 0$, we immediately obtain $\Psi(V'_{i^*}) \ge \Psi(V'_{\max})$,
contradicting our assumption. This proves (a). Part (b) is an immediate consequence
of (a).
\end{proof}
\begin{lemma}\label{l:upsilon_psi}
For any tree $T = (V,E)$ and any $0$-forked subset $L^+ \subseteq V$ we have
$
    \Upsilon(L^+,1) \le 2\Psiinv(L^+)
$.
\end{lemma}
\begin{proof}
Let $\scT_{\max}$ 
be the largest component of $T \setminus L^+$ and
$V_{\max}$ be its node set. Since $L^+$ is a 0-forked query set, $\scT_{\max}$
must be either a 1-hinge-tree or a 2-hinge-tree. Since the only edges that
connect a hinge-tree to external nodes are the edges leading to connection
nodes, we find that $\Gamma(V_{\max},V \setminus V_{\max}) \le 2$.
We can now write
\[
    \Psiinv(L^+)
=
    \max_{\emptyset \not\equiv V' \subseteq V\setminus L^+}\frac{|V'|}{\Gamma(V',V \setminus V')}
\ge
    \frac{|V_{\max}|}{\Gamma(V_{\max},V \setminus V_{\max})}
\ge
    \frac{|V_{\max}|}{2}
=
    \frac{\Upsilon(L^+,1)}{2}
\]
thereby concluding the proof.
\end{proof}
\begin{lemma}\label{l:psi0_psi}
For any tree $T = (V,E)$ and any subset $L \subseteq V$ we have
$
    \Psiinv(L^+) \le \Psiinv(L)
$.
\end{lemma}  
\begin{proof}
Let $V'_{\max}$ be 
any set maximizing $\Psiinv(L^+)$.
Since $V'_{\max} \in V\setminus L^+$, $V'_{\max}$ cannot contain
any node of $L \subseteq L^+$. Hence
\[
    \Psiinv(L)
=
    \max_{\emptyset \not\equiv V' \subseteq V\setminus L}
    \frac{|V'|}{\Gamma(V',V \setminus V')}
\ge
    \frac{|V'_{\max}|}{\Gamma(V'_{\max},V \setminus V'_{\max})}
=
    \Psiinv(L^+)
\]
which concludes the proof.
\end{proof}

We now put together the previous lemmas to show that the query set $L_{\sel}$ minimizes
$\Psiinv$ up to constant factors.
\begin{theorem}\label{th:psi}
For any tree $T = (V,E)$ we have
${\dt \quad
    \Psiinv(L_{\sel})
\le
    \min_{L \subseteq V \,:\, |L| \le \tfrac{1}{4}|L_{\sel}|} 2 \Psiinv(L)
}$.
\end{theorem}
\begin{proof}
Let $L$ be a query set such that $|L| \le |L_{\sel}|/4$.
Then we have the following chain of inequalities:
\[
    \Psiinv(L_{\sel})
\overset{\text{(\lemref{l:psi_upsilon} (b))}}{\le}
    \Upsilon(L_{\sel},1)
\overset{\text{(\lemref{l:upsilon} (b))}}{\le}
    \Upsilon(L^+,1)
\overset{\text{(\lemref{l:upsilon_psi})}}{\le}
    2\Psiinv(L^+)
\overset{\text{(\lemref{l:psi0_psi})}}{\le}
    2\Psiinv(L)~.
\]
In order to apply \lemref{l:upsilon} (b),
we need the condition $|L^+| \le \tfrac{1}{2}|L_{\sel}|$.
This condition holds because, by \lemref{l:deforked},
$|L^+| \le 2|L| \le \tfrac{1}{2}|L_{\sel}|$.
\end{proof}

Finally, as promised, the following corollary contains an interpretable mistake
bound for \pred\ run with a query set returned by \sel.
\begin{corollary}\label{c:promised}
For any labeled tree $(T,\by)$, the number of mistakes
made by \pred\ when run with query set $L^+_{\sel}$ satisfies
\[
    m_{\pred}(L^+_{\sel}, \by)
\le
    4\,\min_{L \subseteq V \,:\, |L| \le \tfrac{1}{8}|L^+_{\sel}|}
    \Psiinv(L)\,\Phi_T(\by)~.
\]
\end{corollary}
\begin{proof}
Observe that \pred\ assigns labels to nodes in $V \setminus L^+_{\sel}$
so as to minimize the resulting cutsize given the labels in the query set
$L^+_{\sel}$. We can then invoke~\cite[Lemma~1]{gb09}, which bounds
the number of mistakes made by the mincut strategy in terms of the
functions $\Psiinv$ and the cutsize. This yields
\[
    m_{\pred}(L^+_{\sel}, \by)
\!\overset{\text{\cite[Lemma~1]{gb09}}}{\le}\!
    2\,\Psiinv(L^+_{\sel})\,\Phi_T(\by)
\overset{\text{(A)}}{\le}
    2\,\Psiinv(L_{\sel})\,\Phi_T(\by)
\!\!\overset{\text{(\thmref{th:psi})}}{\le}\!\!
    4\,\Psiinv(L)\,\Phi_T(\by)~.
\]
Inequality~(A) holds because $L_{\sel} \subseteq L_{\sel}^+$,
and thus $T \setminus L_{\sel}^+$ has connected components of
smaller size than $L_{\sel}$. In order to apply \thmref{th:psi},
we need the conditon $|L| \le \tfrac{1}{4}|L_{\sel}|$, which
follows from a simple combination of \lemref{l:deforked} and our assumptions:
$|L| \le \tfrac{1}{8}|L_{\sel}^+| \le \tfrac{1}{4}|L_{\sel}|$.
\end{proof}
\begin{remark}
%
%
A mincut algorithm exists which efficiently predicts even when the 
query set $L$ is not 0-forked (thereby gaining a factor of 2 in the cardinality
of the competing query sets $L$ -- see Theorem~\ref{t:together} and 
Corollary~\ref{c:promised}). 
This algorithm is a "batch" variant of the TreeOpt algorithm analyzed in~\cite{cgv09}. 
The algorithm can be implemented in such a way that the total time
for predicting $|V|-|L|$ labels is $\scO(|V|)$.
\end{remark}

\newcommand{\selk}{\sel\ensuremath{\star}}

\subsection{Automatic calibration of the number of queries}\label{s:q_m}
A key aspect to the query selection task is deciding when to stop 
asking queries. Since the more queries are asked the less mistakes are made afterwards,
a reasonable way to deal with this trade-off is to minimize
the number of queries issued during the selection phase plus the 
number of mistakes made during the prediction phase.
For a given pair $A = \langle S,P \rangle$ of prediction and selection algorithms, 
we denote by $[q+m]_A$ the sum of queries made by $S$ and prediction mistakes
made by $P$. Similarly to $m_A$ introduced in \secref{s:algo_an}, $[q+m]_A$
has to scale with the cutsize $\Phi_T(\by)$ of the labeled tree $(T,\by)$
under consideration.

As a simple example of computing $[q+m]_A$, consider a line graph $T = (V,E)$.
Since each query set on $T$ is 0-forked, \thmref{th:lb} and 
\corref{th:bto_f} ensure that an optimal strategy for selecting
the queries in $T$ is choosing a sequence of nodes such that
the distance between any pair of neighbor nodes in $L$ is equal.
The total number of mistakes that can be forced on $V\setminus L$ is,
up to a constant factor, $\bigl(|V|/|L|\bigr)\Phi_{T}(\by)$. 
Hence, the optimal value of $[q+m]_A$ is about 
\begin{equation}\label{e:line}
|L|+\frac{|V|}{|L|}\Phi_{T}(\by)~.
\end{equation}
Minimizing the above expression over $|L|$ clearly requires knowledge of 
$\Phi_{T}(\by)$, which is typically unavailable.
In this section we investigate a method for choosing the number of queries
when the labeling is known to be sufficiently regular, that is when a
bound $K$ is known on the cutsize $\Phi_T(\by)$ induced by the adversarial
labeling.\footnote
{
In~\cite{acdfmsz07} a labeling $\by$ of a graph $G$ is said to be
$\alpha$-balanced if, after the elimination of all $\phi$-edges, each
connected component of $G$ is not smaller than $\alpha |V|$ for some
known constant $\alpha \in (0,1)$. In the case of labeled trees,
the $\alpha$-balancing condition is stronger than our regularity
assumption. This is because any $\alpha$-balanced labeling $\by$ implies
$\Phi_T(\by) \leq 1/\alpha-1$. In fact, getting back to the line graph
example, we immediately see that, if $\by$ is $\alpha$-balanced, then
the optimal number of queries $|L|$ is order of $\sqrt{|V|(1/\alpha-1)}$,
which is also $\inf_A [q+m]_A$.
}



We now show that when a bound $K$ on the cutsize is known, a simple
modification of \sel (we call it \selk) exists which optimizes the
$[q+m]_A$ criterion. This means that the combination of \selk\ and \pred\ can
trade-off optimally (up to constant factors) queries against mistakes.

Given a selection algorithm $S$ and a prediction algorithm $P$,
define $[q+m]_{\langle S,P \rangle}$ by
\[
    [q+m]_{\langle S,P \rangle} = \min_{Q \ge 1} \bigl(Q + m_P(L_{S(Q)},K) \bigr)
\]
where $L_{S(Q)}$ is the query set output by $S$ given query budget $Q$, and
$m_P(L_{S(Q)},K)$ is the maximum number of mistakes made by $P$ with query set $L_{S(Q)}$
on any labeling $\by$ with $\Phi_T(\by) \le K$ ---see definition in \secref{s:algo_an}.
Define also $[q+m]_{\opt} = \inf_{S,P} [q+m]_{\langle S,P \rangle}$,
where $\opt = \langle S^*,P^* \rangle$ is an optimal pair of selection
and prediction algorithms.
If \sel\ knows the size of the query set $L^*$ selected by $S^*$, so
that \sel\ can choose a query budget $Q = 8|L^*|$, then a direct
application of \thmref{t:together} guarantees that
$
|L^+_{\sel}| + m_{\pred}(L_{\sel}^+,K) \le 10\, [q+m]_{\opt}
$.
We now show that \selk, the announced modification of \sel, can efficiently
search for a query set size $Q$ such that 
$
    Q + m_{\pred}(L_{\sel(Q)}^+,K) = \scO\bigl([q+m]_{\opt}\bigr)
$
when only $K$, rather than $|L^*|$, is known.
In fact, \thmref{th:lb} and \corref{th:bto_f} ensure that
$m_{\pred}(L_{\sel}^+,K) = \Theta\bigl(\Upsilon(L_{\sel}^+,K)\bigr)$. 
When $K$ is given as side information, \selk\ can operate as follows.
For each $t \le |V|$, the algorithm builds the query set $L^+_t$ and computes 
$\Upsilon(L^+_t,K)$. Then it finds the smallest value $t^*$ minimizing $t + \Upsilon(L^+_t,K)$
over all $t \le |V|$, and selects $L_{\selk} \equiv L_{t^*}$.
We stress that the above is only possible because the algorithm can estimate
within constant factors its own future mistake bound (\thmref{th:lb} and \corref{th:bto_f}),
and because the combination of \sel\ and \pred\ is competitive against all query sets
whose size is a constant fraction of $|L^+_{\sel}|$ ---see \thmref{t:together}.
Putting together, we have shown the following result.
\begin{theorem}
For all trees $(T,\by)$, for all cutsize budgets $K$, and for all labelings $\by$
such that $\Phi_T(\by) \le K$, the combination of \selk\ and \pred\ achieves
$
    |L_{\selk}| + m_{\pred}(L_{\selk}^+,K) = \scO\bigl([q+m]_{\opt}\bigr)
$
when $K$ is given to \selk\ as input.
\end{theorem}
Just to give a few simple examples of how \selk\ works, consider a star graph. 
It is not difficult to see that in this case $t^*=1$ {\em independent} of $K$, 
i.e., \selk\ always selects the center of the star, which is intuitively the optimal choice.
If $T$ is the ``binary system'' mentioned in the introduction, then $t^*=2$
and \selk\ always selects the centers of the two stars, again independent of $K$.
At the other extreme, if $T$ is a line graph, then \selk\ picks the query nodes
in such a way that the distance between two consecutive nodes of $L$ in $T$ is 
(up to a constant factor) equal to $\sqrt{|V|/K}$. Hence $|L| = \Theta(\sqrt{|V|K})$, 
which is the minimum of~(\ref{e:line}) over $|L|$ when $\Phi_{T}(\by) \leq K$.

\section{On the prediction of general graphs}\label{s:graph}
In this section we provide a general lower bound for prediction on arbitrary
labeled graphs $(G,\by)$. We then contrast this lower bound to some results
contained in Afshani et al.~\cite{acdfmsz07}. 

Let $\Phi^R_G(\by)$ be the sum of the effective resistances~(see, e.g., \cite{LP09})
on the $\phi$-edges of $G = (V,E)$. The theorem below shows that any prediction
algorithm using any query set $L$ such that $|L| \le \tfrac{1}{4}|V|$ makes at
least order of $\Phi^R_G(\by)$ mistakes.
This lower bound holds even if the algorithm is allowed to use a randomized
adaptive strategy for choosing the query set $L$, that is, a randomized strategy
where the next node of the query set is chosen after receiving the
labels of all previously chosen nodes.
\begin{theorem}
\label{th:lb_graphs}
Given a labeled graph $(G,\by)$, for all $K \le |V|/2$, 
there exists a randomized labeling strategy such that
for all prediction algorithms $A$ choosing a query set of size $|L| \le \tfrac{1}{4}|V|$
via a possibly randomized adaptive strategy, the expected number of mistakes made
by $A$ on the remaining nodes $V \setminus L$ is at least $K/4$, while $\Phi^R_G(\by) \le K$.
\end{theorem}
%
%
The above lower bound (whose proof is omitted) appears to contradict
an argument by Afshani et al.~\cite[Section~5]{acdfmsz07}. This argument establishes
that for any $\varepsilon > 0$ there exists a randomized algorithm using at most
$K\ln(3/\varepsilon) + K\ln(|V|/K) + \scO(K)$ queries on any given graph $G = (V,E)$
with cutsize $K$, and making at most $\varepsilon|V|$ mistakes on the remaining
vertices. This contradiction is easily obtained through the following simple
counterexample: assume $G$ is a line graph where all node labels are $+1$ but
for $K = o\bigl(|V|/\ln|V|\bigr)$ randomly chosen nodes, which are also given random labels.
For all $\varepsilon = o\bigl(\tfrac{K}{|V|}\bigr)$, the above argument implies
that order of $K\ln|V| = o(|V|)$ queries are sufficient to make at most
$\varepsilon|V| = o(K)$ mistakes on the remaining nodes, among which $\Omega(K)$
have random labels ---which is clearly impossible.

\section{Efficient Implementation}
\label{s:impl}
In this section we describe an efficient implementation of \sel\ 
and $\pred$.
We will show that the total time needed for selecting $Q$ queries is $\scO(|V| \log Q)$,
the total time for predicting $|V|-Q$ nodes is  $\scO(|V|)$, and that
the overall memory space is again $\scO(|V|)$.

In order to locate the largest subtree of $T \setminus L_{t-1}$, the algorithm
maintains a priority deque~\cite{kv03} $D$ containing at most $Q$ items.
This data-structure enables to find and eliminate the item
with the smallest (resp., largest) key in time $\scO(1)$ (resp., time $\scO(\log Q))$. 
In addition, the insertion of a new element takes time $\scO(\log Q)$.

Each item in $D$ has two records: a reference to a node in $T$ and the
priority key associated with that node. 
Just before the selection of the\footnote
{
If $t=1$ the priority deque $D$ is empty.
} 
$t$-th query node $i_t$, the $Q$ references point to nodes contained in the
$Q$ largest subtrees in $T \setminus L_{t-1}$, while the corresponding keys
are the sizes of such subtrees.
Hence at time $t$ the item $top$ of $D$ having the largest key points to a 
node in $T^t_{\max}$.

First, during an initialization step, \sel\ creates, for each edge $(i,j) \in E$, 
a directed edge $[i,j]$ from $i$ to $j$ and the twin directed edge $[j,i]$ from
$j$ to $i$. During the construction of $L_{\sel}$ the algorithm also stores 
and maintains the current size $\sigma(D)$ of $D$, i.e., the total
number of items contained in $D$.
We first describe the way $\sel$ finds node $i_t$ in $T^t_{\max}$. Then
we will see how $\sel$ can efficiently augment the query set $L_{\sel}$ to
obtain $L^+_{\sel}$.

Starting from the node $r$ of $T^t_{\max}$ referred to by\footnote
{
In the initial step $t=1$ (i.e., when $T^t_{\max} \equiv T$) node $r$ can be
chosen arbitrarily .
} 
$D$,
\sel\ performs a depth-first visit of $T^t_{\max}$, followed by the elimination of the item
with the largest key in $D$. For the sake of simplicity,
consider $T^t_{\max}$ as rooted at node $r$.
Given any edge $(i,j)$, we let $T_i$ and $T_j$ be the two subtrees 
obtained from $T^t_{\max}$ after removing edge $(i,j)$, where $T_i$ contains
node $i$, and $T_j$ contains node $j$.
During each backtracking step of the depth-first visit from a node $i$ to a node $j$, 
\sel\ stores the number of nodes $|T_i|$ contained in $T_i$. 
This number gets associated with $[j,i]$. Observe that this task can be accomplished 
very efficiently, since $|T_i|$ is equal to $1$ plus the number of nodes
of the union of $T_{c(i)}$ over all children $c(i)$ of $i$. 
These numbers can be recursively calculated by summing the size values 
that \sel\ associates with all direct edges $[i,c(i)]$ in the previous backtracking
steps. Just after storing the value $|T_i|$, the algorithm also
stores $|T_j|=|T^t_{\max}|-|T_i|$ and associates this value with the twin directed edge $[i,j]$. 
The size of $T^t_{\max}$ is then stored in $D$ as the key record of the 
pointer to node $r$.

It is now important to observe that the quantity $\sigma(T^t_{\max},i)$ 
used by $\sel$ (see Section \ref{s:algo}) is simply
the largest key associated with the directed edges $[i,j]$ 
over all $j$ such that $(i,j)$ is an edge of $T^t_{\max}$.
Hence,  a new depth-first visit is enough to find 
in time $\scO(|T^t_{\max}|)$ the $t$-th node
$i_t=\arg\min_{i \in T^t_{\max}}\sigma(T^t_{max},i)$
selected by $\sel$.
Let $N(i_t)$ be the set of all nodes adjacent to node $i_t$ in $T^t_{\max}$. 
For all nodes $i' \in N(i_t)$, \sel\ compares $|T_{i'}|$ to the smallest key 
$bottom$ stored in $D$. We have three cases:
%
\begin{enumerate}
\item If $|T_{i'}| \le bottom$ and $\sigma(D) \ge Q-t$
then the algorithm does nothing, since $T_{i'}$ (or subtrees thereof) 
will never be largest in the subsequent steps of the construction of $L_{\sel}$, 
i.e., there will not exist any node $i_{t'}$ with $t'>t$ such that $i_{t'} \in T_{i'}$.
\item If $|T_{i'}| \le bottom$ and $\sigma(D) < Q-t$, or if
$|T_{i'}| > bottom$ and $\sigma(D) < Q$ then
\sel\ inserts a pointer to $i'$ together with the associated
key $|T_{i'}|$. Note that, since $D$ is not full (i.e., $\sigma(D) < Q)$,
the algorithm need not eliminate any item in $D$.
\item If $|T_{i'}| > bottom$ and $\sigma(D) = Q$ then \sel\ eliminates 
from $D$ the item having the smallest key,
and inserts a pointer to $i'$, together with the associated key $|T_{i'}|$.  
\end{enumerate}
%
Finally, \sel\ eliminates node $i_t$ and all edges (both undirected and directed)
incident to it. Note that this elimination implies that we
can easily perform a depth-first visit within $T^s_{\max}$ for each $s \le Q$,
since $T^s_{\max}$ is always completely disconnected from the rest
of the tree $T$.

In order to turn $L_{\sel}$ into $L^+_{\sel}$, the algorithm proceeds incrementally,
using a technique borrowed from~\cite{cgv09}.
Just after the selection of the first node $i_1$, 
a depth-first visit starting from $i_1$ is performed. 
During each backtracking step of this visit, the algorithm associates with
each edge $(i,j)$, the closer node to $i_1$ between the two nodes $i$ and $j$. 
In other words, \sel\ assigns a direction to each undirected edge $(i,j)$ so as to be able 
to efficiently find the path connecting each given node $i$ to $i_1$. 
When the $t$-th node $i_t$ is selected, \sel\ 
follows these edge directions from $i_t$ towards $i_1$.
Let us denote by $\pi(i,j)$ the path connecting node $i$ to node $j$. 
During the traversal of $\pi(i_1,i_t)$, the algorithm assigns a special
mark to each visited node, until the algorithm reaches the first node 
$j \in \pi(i_1,i_t)$ which has already been marked. 
Let $\eta(i,L)$ be the maximum number of edge disjoint paths connecting
$i$ to nodes in the query set $L$. Observe that all nodes $i$
for which $\eta(i,L_t) > \eta(i,L_{t-1})$ must necessarily belong
to $\pi(i_t,j)$. We have $\eta(i_t,L_t)=1$, and $\eta(i,L_t)=2$, for all
internal nodes $i$ in the path $\pi(i_t,j)$. Hence, $j$ 
is the unique node that we may need to add as a new fork node 
(if $j \not\in \fork(L_{t-1})$). 
In fact, $j$ is the unique node such that the number of edge-disjoint
paths connecting it to query nodes may increase, and be actually larger than $2$. 

Therefore if $j \in L^+_{t-1}$ we need not add any fork node during the 
incremental construction of $L^+_{\sel}$. 
On the other hand, if $j \not\in L^+_{t-1}$ then $\eta(i,L_{t-1}) = 2$, 
which implies $\eta(i,L_{t})=3$. This is the case when $\sel$ views $j$
as new fork node to be added to the query set $L_{\sel}$ under consideration.

In order to bound the total time required by \sel\ for selecting 
$Q$ nodes, we rely on \lemref{l:split}, showing that
$|T^t_{\max}| \leq 2|V|/t$. 
The two depth-first visits performed for each node $i_t$ take
$\scO(|T^t_{\max}|)$ steps. Hence the overall running time spent on the depth-first visits
is $ \scO(\sum_{t \le Q} 2|V|/t) = \scO(|V| \log Q)$. The total time spent for incrementally finding 
the fork nodes of $L_{\sel}$ is linear in the number of nodes marked by the 
algorithm, which is equal to $|V|$. 
Finally, handling the priority deque $D$ takes $|V|$ times the worst-case time 
for eliminating an item with the smallest (or largest) key or adding a new item.
This is again $\scO(|V| \log Q)$.

We now turn to the implementation of the prediction phase. 
\pred\ operates in two phases. 
In the first phase, the algorithm performs a depth-first visit of each hinge-tree $\scT$, 
starting from each connection node (thereby visiting the nodes of all 1-hinge-tree once,
and the nodes of all 2-hinge-tree twice).
During these visits, we add to the nodes a tag containing (i) the label of node $i_{\scT}$ 
from which the depth-first visit started, and (ii) the distance between $i_{\scT}$ and the
currently visited node. 
In the second phase, we perform a second depth-first visit, this time on the whole tree $T$.
During this visit, we predict each node $i \in V \setminus L$
with the label coupled with smaller distance stored in the tags of\footnote
{
If $i$ belongs to a 1-hinge-tree, we simply predict $y_i$ with the unique 
label stored in the tag.
} 
$i$.
The total time of these visits is linear in $|V|$ since each node of $T$ gets 
visited at most $3$ times.

\section{Conclusions and ongoing work}\label{s:concl}
The results proven in this paper characterize, up to constant factors,
the optimal algorithms for adversarial active learning on trees in two
main settings. In the first setting the goal is to minimize the number of
mistakes on the non-queried vertices under a certain query budget. In
the second setting the goal is to minimize the sum of queries and mistakes
under no restriction on the number of queries.

An important open question is the extension of our results to the general
case of active learning on graphs. While a direct characterization of
optimality on general graphs is likely to require new analytical tools,
an alternative line of attack is reducing the graph learning problem to the
tree learning problem via the use of spanning trees. Certain types of
spanning trees, such as random spanning trees, are known to summarize
well the graph structure relevant to passive learning ---see, e.g.,
\cite{cgv09,CGVZ10a,VCGZ12}.
In the case of active learning, however, we want good query sets on the graph
to correspond to good query sets on the spanning tree, and random spanning
trees may fail to do so in simple cases. For example, consider a set of
$m$ cliques connected through bridges, so that each clique is connected
to, say, $k$ other cliques. The breadth-first spanning tree of this graph
is a set of connected stars. This tree clearly reveals a query set
(the star centers) which is good for regular labelings (cfr., the binary
system example of \secref{s:intro}). On the other hand, for certain choices
of $m$ and $k$ a random spanning tree has a good probability of hiding
the clustered nature of the original graph, thus leading to the selection of
bad query sets.

In order to gain intuition about this phenomenon, we are currently running
experiments on various real-world graphs using different types of spanning
trees, where we measure the number of mistakes made by our algorithm (for 
various choices of the budget size) against common baselines.

We also believe that an extension to general graphs of our algorithm does actually 
exist.
However, the complexity of the methods employed in~\cite{gb09} suggests that
techniques based on minimizing $\Psiinv$ on general graphs are computationally
very expensive.

Finally, it would be interesting to combine active learning techniques on
the nodes of a graph with those for predicting links (e.g., \cite{cgvz12a,cgvz12b}).

\smallskip\noindent
\textbf{Acknowledgments.}
This work was supported in part
by Google Inc. through a Google Research Award and
by the PASCAL2 Network of Excellence under EC grant no.\ 216886.
This publication only reflects the authors' views.


\end{document}